\documentclass[smallcondensed]{svjour3}
\smartqed 

\usepackage{stmaryrd}
\usepackage{bussproofs}
\usepackage{multirow}

\usepackage{mathptmx}

\usepackage{latexsym}
\usepackage{amsfonts}
\usepackage{amssymb}
\usepackage{amsmath}
\usepackage{graphicx}
\usepackage{color}

\usepackage{wrapfig,lipsum,booktabs}
\usepackage{semantic}
\usepackage{listings}
\usepackage{proof}

\usepackage{subcaption}
\usepackage{url}
\usepackage{wrapfig}

\usepackage{algorithm}
\usepackage{algpseudocode}

\newcommand{\N}{\mathbb{N}}

\definecolor{codegreen}{rgb}{0,0.6,0}
\definecolor{codegray}{rgb}{0.5,0.5,0.5}
\definecolor{codepurple}{rgb}{0.58,0,0.82}
\definecolor{backcolour}{rgb}{0.95,0.95,0.92}

\lstdefinestyle{csp}{
    backgroundcolor=\color{backcolour},   
    commentstyle=\color{codegreen},
    keywordstyle=\color{magenta},
    numberstyle=\tiny\color{codegray},
    stringstyle=\color{codepurple},
    basicstyle=\footnotesize,
    breakatwhitespace=false,         
    breaklines=true,                 
    captionpos=b,                    
    keepspaces=true,                 
    numbers=left,                    
    numbersep=5pt,                  
    showspaces=false,                
    showstringspaces=false,
    showtabs=false,                  
    tabsize=4,
    morekeywords={var,define,Stop,Skip,if,else,assert}
}

\begin{document}


\title{GRAVITAS: A Model Checking Based Planning and Goal Reasoning Framework for Autonomous Systems}  

%

\author{Hadrien Bride \and
        Jin Song Dong \and
        Ryan Green \and
        Zh\'e H\'ou \and
        Brendan Mahony \and
        Martin Oxenham 
}

\institute{
Hadrien Bride \at
Griffith University \\
\email{h.bride@griffith.edu.au}           
\and
Jin Song Dong \at
National University of Singapore\\
Griffith University\\
\email{j.dong@griffith.edu.au}
\and
Ryan Green \at
Defence Science and Technology Group, Australia\\
University of South Australia\\
\email{ryan.green@adelaide.edu.au}
\and 
Zh\'e H\'ou \at 
Griffith University \\
\email{z.hou@griffith.edu.au}   
\and 
Brendan Mahony \at 
Defence Science and Technology Group, Australia\\
\email{Brendan.Mahony@dst.defence.gov.au} 
\and 
Martin Oxenham \at 
Defence Science and Technology Group, Australia \\
\email{Martin.Oxenham@dst.defence.gov.au} 
}

\maketitle

\begin{abstract}
  While AI techniques have found many successful applications in
  autonomous systems, many of them permit behaviours that are difficult to
  interpret and may lead to uncertain results. We follow the ``verification as
  planning'' paradigm and propose to use model checking techniques to solve
  planning and goal reasoning problems for autonomous systems. We give a new
  formulation of Goal Task Network (GTN) that is tailored for our model checking
  based framework. We then provide a systematic method that models GTNs in the
  model checker Process Analysis Toolkit (PAT). We present our planning and goal
  reasoning system as a framework called Goal Reasoning And Verification for
  Independent Trusted Autonomous Systems (GRAVITAS) and discuss how it helps
  provide trustworthy plans in an uncertain environment. Finally, we
  demonstrate the proposed ideas in an experiment that simulates a survey mission
  performed by the REMUS-100 autonomous underwater vehicle.    
\end{abstract}



\section{Introduction}
\label{sec:intro}

Planning is a central and hard Artificial Intelligence problem that is essential in the development of autonomous systems. Many existing solutions require a
controlled environment in order to function correctly and reliably. However,
there are situations where adaptive autonomous systems are required to run for
a long period of time and cope with uncertain events during the deployment. 
Our work is motivated by the requirements of next generation
autonomous underwater vehicles (AUV) in law enforcement and defence industries. 
Particularly, we are currently developing a decision making system which is suitable for an AUV 
designed to stay underwater for a long time and to have very limited communication
with the outside world. The AUV is expected to carry out survey missions on its
own and report details of its surveillance at semi-regular intervals. During
the mission, the AUV may encounter underwater currents, deep ocean terrain,
fishing boats, objects and places of interest, hostile vehicles etc., each of
which may affect its ability to achieve its goals. The AUV must be able to
decide which goals to pursue when such dynamic events occur and plan tasks to
achieve the goals in an agile manner.

When there are uncertainties in the environment, planning becomes an
even harder problem. In this case, the agent's goal may be affected,
thus both selecting a new goal and re-planning are necessary. This
generally follows a \emph{note-assess-guide} procedure, where
\emph{note} detects discrepancies (e.g.,~\cite{PCP13}),
\emph{assess} hypothesises causes for discrepancies, and
\emph{guide} performs a suitable response~\cite{AOC06}. Differing
from classical planning where the goal is fixed, when a discrepancy
is detected, it is often necessary to change the current goal. Goal
reasoning is about selecting a suitable goal for the planning
process. There have been various formalisms that attempt to solve
planning problems in a dynamic environment, including hierarchical
planning methods, such as hierarchical task networks
(HTN)~\cite{erol1994umcp} and hierarchical goal networks
(HGN)~\cite{shivashankar2012hierarchical}, and goal reasoning systems
such as the Metacognitive Integrated Dual-Cycle Architecture
(MIDCA)~\cite{Cox2016} and the goal lifecycle
model~\cite{RobertsAJAWA15,Johnson2016}.


Although some of the above formalisms have been successfully applied to solve
real life problems, the verification aspect of the problem remains to be
addressed. Usually, planning is solved by heuristic search, but heuristics may miss some cases and produce sub-optimal or even undesired results. The correctness, safety, and
security issues of autonomous systems are particularly important in
mission-critical use cases. To tackle this problem, we
turn to formal methods, which have been used to solve planning problems in
the literature. For example, Giunchiglia et al. proposed to solve planning
problems using model checking~\cite{giunchiglia1999planning}; Kress-Gazit et
al.'s framework translates high-level tasks defined in linear temporal logic
(LTL)~\cite{Pnueli1977} to hybrid controllers~\cite{KressGazit2009}; 
Bensalem et al.~\cite{bensalem2014verification} used verification and
validation (V\&V) methods to solve planning.

Following the above ideas, we propose a model checking based framework for
hierarchical planning and goal reasoning. Model checking is a technique that can automatically verify whether certain properties are satisfied in a model using exhaustive search. Model checking is especially strong in addressing uncertain and concurrent behaviours. It has been successfully applied to modelling and verification of uncertain environments such as network attacks which may involve arbitrary behaviour in communication protocols~\cite{bai2013authscan}. We choose to use the model checker Process Analysis Toolkit
(PAT)~\cite{SLD09} -- a self-contained tool that supports
composing, simulating and reasoning about concurrent, probabilistic and timed
systems with non-deterministic behaviours. Besides, we choose PAT because 
its verification outcome includes a witness trace which can be effectively extracted to form a plan.

Since our planning method is realised in PAT,
we can formulate inconsistency and incompatibility of plans and goals as reachability and LTL properties~\cite{clarke1986automatic}, and verify them at execution time.
For instance, when a new goal is generated during execution,
we can check whether the new goal conflicts with existing goals, and
select a subset of goals that are compatible with each other. 
In addition, we can also verify the planning model itself such that a given
planning model does not output plans that may lead to undesired
events. 
Based on the use of PAT, we propose a novel planning and goal reasoning system called Goal Reasoning And Verification for Independent Trusted Autonomous Systems (GRAVITAS) -- a framework in charge of producing and actuating verifiable and explainable plans for autonomous systems.  
We demonstrate the proposed techniques in a simulation environment which is compatible with modern AUVs.
\section{Preliminaries}
\label{sec:pre}


Model checking~\cite{Clarke2000} is an automated technique for formally verifying
finite-state systems. In model checking,
specifications of finite-states systems, i.e., properties to be
verified, are often expressed in temporal logic
whereas the system to be checked is modelled as a state transition
graphs. Model checking involves a search procedure which is used to
determine whether the model can reach a state that satisfies the
specifications. We briefly introduce the modelling language and the specification language in PAT below.



\paragraph{Modelling language.} Models that can be verified using PAT~\cite{SLD09}
may take several forms, including: CSP\# models,
timed automata, real-time models and probabilistic models. The latter ones are extensions of the CSP\# language.
In this paper, all the examples only use CSP\# -- a high-level
modelling language that extends Tony Hoare's Communicating Sequential Processes~\cite{hoare1978communicating} with C\#. Formally, a CSP\# model is a
tuple $\langle \mathit{Var}, \mathit{init}_G, P \rangle$ where
$\mathit{Var}$ is a finite set of global variables, $\mathit{init}_G$
is the initial valuation of global variables, and $P$ is a process.
Variables are typed: either by a pre-defined type (e.g., boolean,
integer, array) or by any user-defined data type. If the type of a
variable is not explicitly stated, then, by default, the variable is
assumed to be an integer. For instance, an integer variable
$v$ and an integer array $a$ can be defined respectively as follows:

\begin{lstlisting}
var v = 0;
var a[3]:{0..5} = [0(3)];
\end{lstlisting}

\noindent The range of a variable can be specified in the definition. For instance, the
annotation `$:\{0..5\}$' specifies that the value of each
element in $a$ must be in the close interval $[0,5]$. The three values of $a$ are initialised as
$0$s, as denoted by right-hand side of $a$'s declaration -- i.e.,
$[0(3)]$ is equivalent to $[0,0,0]$. 


A CSP\# process is defined using the following syntax:

\begin{lstlisting}[mathescape]
P($x_1,x_2,...$) = Exp;
\end{lstlisting}

\noindent where $P$ is the process name, $x_1,x_2,...$ are the
optional parameters of the process, and $Exp$ is a process
expression, which defines the computation of the process. The running example in
this paper uses the following \emph{subset} of CSP\#, shown in Backus–Naur form:
\begin{lstlisting}[mathescape]
Exp ::=  Stop | Skip | Ev{Prog} $->$ Exp | Exp ; Exp | Exp || Exp 
		| Exp [] Exp | Exp <> Exp | if (Cond) {$Prog_1$} else {$Prog_2$} 
		| [Cond] Exp
\end{lstlisting}
Interested readers can refer to Sun et al.'s paper~\cite{sun2008} for the complete syntax and semantics of CSP\#. The process $Stop$ terminate the
execution of a process. The process $Skip$ does nothing. Let P and Q be CSP\# processes. The process
expression $e\{Prog\} -> P$ first activates the event labelled by $e$ and
executes the statements given by $Prog$, then it proceeds with the execution of
$P$. The statements of $Prog$ are defined by the syntax and semantics of C\#
and can therefore manipulate complex data types. The processes $P;Q$ and $P||Q$
respectively express the sequential and parallel composition of processes $P$
and $Q$. We use $P[]Q$ to state that either $P$ or $Q$ may
execute, depending on which one performs an event first. On the other hand,
$P<>Q$ non-deterministically executes either $P$ or $Q$. The expression $if$
$(Cond)$ $Prog_1$ $else$ $Prog_2$ is self-explanatory.
Finally, the expression $[Cond] \ P$, where $Cond$ is a boolean
expression, defines a guarded process such that $P$ only executes when 
$Cond$ is satisfied.

\paragraph{Specification language.} We can check whether a CSP\# process $P$ satisfies a given specification using
the following expression:

\begin{lstlisting}[mathescape]
#assert P($x_1,x_2,...$) property;
\end{lstlisting}

\noindent where $property$ can be $deadlockfree$ (the process does
progress until reaching a terminating state), $divergencefree$ (the
process performs internal transitions forever without engaging any useful
events), $deterministic$ (the process does not involve non-deterministic
choices), and $nonterminating$ (no terminating states can be reached).

Also, we can check whether the transition system can reach a state where a boolean expression $Cond$ is satisfied using:
\begin{lstlisting}[mathescape]
#assert P($x_1,x_2,...$) reaches Cond;
\end{lstlisting}

\noindent Additionally, we can check whether a process $P$ satisfies a LTL (cf. Huth et al.'s book~\cite[Section~3.2.1]{Huth2004}) formula $F$ using:

\begin{lstlisting}[mathescape]
#assert P($x_1,x_2,...$) |= F;
\end{lstlisting}



\paragraph{PAT output.} When checking LTL properties, PAT produces a counter-example when the
property to be checked cannot be satisfied, and only outputs ``yes'' when when
the property can be satisfied. For reachability properties, which are
widely-used in the planning technique of this paper, PAT outputs different
information. When the desired states cannot be reached, PAT outputs ``no''.
When the desired states can be reached, PAT produces a witness trace of actions
that leads to the desired states. When model checking reachability properties,
the user can specify one of the following \emph{verification engines}: If a breath-first search based
engine is used and the desired states can be reached, then PAT will output the
shortest witness trace, which is useful when finding certain ``optimal'' plans.
Furthermore, the user can tell PAT to output the witness trace that optimises
certain criterion. For example, the following code will produce witness traces
that respectively yield maximum reward and minimum penalty, assuming that
$Cond$ is reachable and $reward$ and $penalty$ are 
predefined variables:

\begin{lstlisting}[mathescape]
#assert P($x_1,x_2,...$) reaches Cond with max(reward);
#assert P($x_1,x_2,...$) reaches Cond with min(penalty);
\end{lstlisting}
\section{Motivating Example}
\label{sec:example}

Surveying underwater areas and reporting
back the locations of potential objects of interest are important usages of AUVs. For instance, in the
search for the missing aircraft from Malaysia Airlines Flight 370, AUVs
were deployed in deep ocean areas to locate debris of the
aircraft~\cite{mh370search}. There are also demands and interests from the defence
industry to demonstrate the abilities to scan underwater areas for naval mines
and dumped arms, as shown in the Wizard of Aus Autonomous Warrior Trial
~\cite{wizardofaus2018}.

In this paper we run an example with the following context that demonstrates a
common survey mission for the AUV: the AUV is to be deployed at the
\emph{initial position} and to be recovered at the \emph{final position}.
During the mission, the AUV is expected to scan two \emph{survey areas} and
record the locations of \emph{objects of interest} upon identification.
Although our technique is general and could be used on all forms of AUVs and UAVs, we
specifically target a torpedo-shaped AUV named REMUS-100~\cite{remus100}, which
is equipped with side scanners that are able to detect surrounding objects. The
side scanners have a scan range of about 15 meters and therefore, in order to
cover large area, the AUV should perform a \textit{lawn mowing} pattern so that the
survey area is fully covered. The overall mission is visualised in
Figure~\ref{fig:overall_mission}.

\begin{figure}[ht!]
    \centering
    \includegraphics[width = 0.6\textwidth]{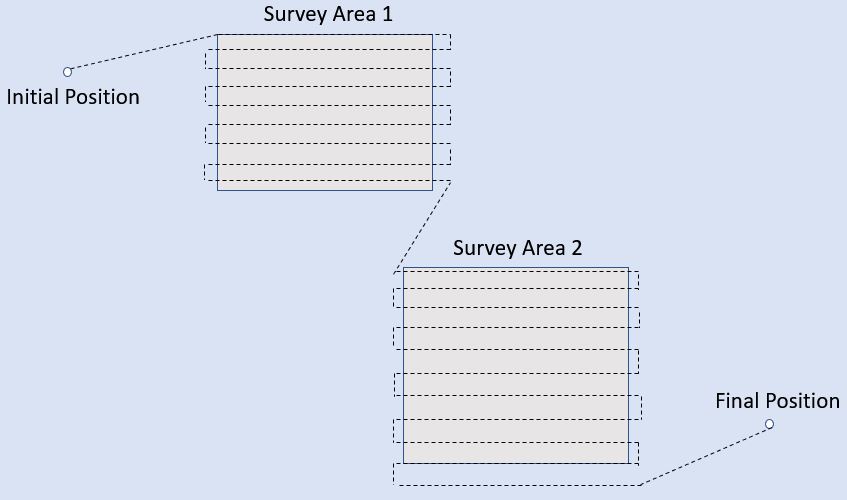}
    \caption{An illustration of the overall survey mission.}
    \label{fig:overall_mission}
\end{figure}

To deal with uncertainties of the environment,
such as changes of survey areas, unexpected events during the transit from one
location to another, our technique must be agile enough to accommodate the
dynamics of the environment. The AUV also needs to make smart decisions
autonomously, these include the order in which to visit the survey areas and
the entry and exit point of each survey area that maximize trajectory
efficiency.
\section{Planning and Goal Reasoning via PAT}
\label{sec:plan}

This section discusses how to solve Goal Task Network planning problems using model
checking. 
We first give a new formalism of the GTN that is suitable for modelling in CSP
\#. We then propose a model checking based approach to model GTN and solve the planning problem. We also discuss how goal selection -- a vital aspect of goal reasoning -- can be done in this approach.

\subsection{Goal Task Networks}
\label{subsec:pre_gtn}

Goal task networks (GTNs) are an extension and unification of
hierarchical task networks and hierarchical goal
networks~\cite{Shivishankarthesis,ASR16}. The main conceptual
advantage of hierarchical task networks (HTNs), when compared to flat-structured
task networks, is their ability to describe dependencies among
actions in the form of nested task networks. HTNs have an
explicit task hierarchy which generally
reflects the hierarchical structure of many real-world planning
applications. This hierarchy has decomposition methods which can then be used during the planning
phase following the well known \emph{divide and conquer} scheme. Due
to this, HTNs planners are much more scalable and performant than
classical planners in practice if the hierarchy is well-designed.

Goal task networks of Alford et al.~\cite{ASR16} are similar to hierarchical task networks but also consider goals and sub-goals in addition to tasks and sub-tasks. As a result, they inherit the advantages of HTNs but also provide flexibility and reasoning capabilities in goal reasoning. We give an adaptation of the original GTN below with a focus on guarded state transitions, which are in the same form as processes in CSP\#.

Let $\mathcal{V} = \{v_0, \cdots, v_{d-1} \}$ be a finite set of variables. Without loss of generality, the state $s$ of a goal task network over $\mathcal{V}$ is defined as a function $s : \mathcal{V} \to \N$ assigning a non-negative integer to each variable of $\mathcal{V}$. 
The set of goal task networks $\mathcal{E}$ is recursively defined as $e \in \mathcal{E} \Leftrightarrow e = \langle E_e, \mathrm{g}_e, \tau_e\rangle$ where: 
\begin{itemize}
	\item $E_e \subseteq \mathcal{E}$ is a finite set of sub-tasks/goals,
	\item $\mathrm{g}_e : (\mathcal{V} \to \N) \to \{\bot, \top\}$ is the guard associated with $e$, and
	\item $\tau_e :  (\mathcal{V} \to \N) \to (\mathcal{V} \to \N)$ is the state transition function associated with $e$.
\end{itemize}

Let $e = \langle E_e, \mathrm{g}_e, \tau_e\rangle$ be a goal task
network. Then $e$ can conceptually represent a task or a goal. In the sequel we shall loosely refer to $e$ as a task or a goal when the context is clear. When
$e$ is a task, its guard models the conditions necessary for the task
to begin. When $e$ is a goal, its guard models the conditions under
which the goal is achieved.

Goal task networks whose set of sub-tasks/goals is empty are called \textit{primitive tasks/goals} and describe the elementary block of goal task network executions.  

The state of a goal task network evolves during its execution
according to the following \textit{firing rules}: A task/goal $e$ is
\textit{enabled} in state $s$ if and only if $\mathrm{g}_e(s) =
\top$. A task/goal enabled in state $s$ can be \textit{fired}, when
it does so, it leads to a new state $s' = \tau_e(s)$.

If $e$ is a primitive task/goal then $s \xrightarrow{e} s'$ denote
the fact that $e$ is enabled in state $s$ and that its firing leads
to state $s'$. If $e$ is not a primitive task/goal then
$s \xrightarrow{e} s'$ denote the fact that there exists a
\textit{valid execution} of $e$ starting in state $s$ and leading to
state $s'$. 

Given an initial state $s_0$, a valid execution of $e$ is a sequence
$e_0, \cdots, e_{n-1} ,e$, of tasks/ goals, where $n \in \N$, such
that $\{e_0, \cdots, e_{n-1}\} \subseteq E_e$ and $s_0
\xrightarrow{e_0} \cdots \xrightarrow{e_{n-1}} s_n \xrightarrow{e}
s_{n+1}$. The set of all valid execution starting from a given state $s$
is denoted by $\Sigma_s$.

A \textit{GTN planning problem} is tuple $P = \langle e, i\rangle$ where $e$ is goal task network and $i$ is the initial state of $e$. The set of solutions for $P$ is the the set of all valid plans $\Sigma_i$, i.e., the set of all valid executions of $e$ starting in state $i$.

A formalised model of our GTN definitions in Isabelle/HOL is available online~\footnote{\url{https://figshare.com/articles/GTN_thy/6964394}}.
The following theorem establishes that our GTN formalism can be used to represent the GTN of Alford et al~\cite{ASR16}. The other direction is not important in the discussion of this paper.

\begin{theorem}
Given a GTN $(I, \prec, \alpha)$ in Alford et al.'s notation~\cite{ASR16} where $I$ is the set of goals and tasks, $\prec$ is a preorder between goals and tasks, and $\alpha$ is a set of labels/names of goal/task instances, there is a corresponding GTN $\langle E, \mathrm{g}, \tau\rangle$ in the above definition.
\end{theorem}

\subsection{Translating GTN Into CSP\#}
\label{subsec:trans}


Let $e = \langle E_e, \mathrm{g}_e, \tau_e \rangle$ and $\{e_0, \cdots, e_{n-1}\} \subseteq E_e$ be GTNs defined over the set of variables $\mathcal{V} = \{v_0, \cdots, v_{d-1} \}$. Further, let $i$ be the initial state of $e$. The GTN planning problem $P = \langle e, i\rangle$ is modelled as follows.

First, the variables are declared and initialised to their initial values.
\begin{lstlisting}[mathescape=true]
var v0 = i(v0);$\cdots$
var v$_{d-1}$ = i(v$_{d-1}$);
\end{lstlisting}

Second, the GTN $e$ and its sub-GTNs, as well as their sub-GTNs, are recursively defined using to the following template. 
\begin{lstlisting}[mathescape=true]
trans$_{e}$() = [g$_e$] event$_{e}$ {$\tau_e$} $->$ Skip
sub$_{e}$() = e$_0$() <> $\cdots$ <> e$_{n-1}$() ; (sub$_{e}$() <> Skip)
e() = sub$_{e}$() ; trans$_{e}$()
\end{lstlisting}

The process $trans_{e}()$ is guarded by the boolean predicate $\mathrm{g}_e$
and, if executed, transforms the current state into its successor state
according to the state transition $\tau_e$. Note that $\tau_e$ can be
effectively encoded by the set of C\# statements as C\# is Turing complete. Further, the process $sub_{e}()$ models the set of $e$'s sub-GTNs
that can be executed before executing $trans_{e}()$. In $sub_e()$, we use
non-deterministic choices $<>$ to connect the execution of sub-GTNs. The last
part of $sub_e$, i.e., $(sub_e() <> skip)$, allows the process to repeat zero or
more times, which effectively chooses and executes any sub-GTN zero or more times.
This general translation allows the execution of GTN $e$ to be decomposed to
the execution of any subset of sub-GTNs $\{e_0,\cdots,e_{n-1}\}$ for any number
of times. Finally, the CSP\# process $e()$ links the processes $sub_{e}()$ and
$trans_{e}()$ to model the behaviour of the GTN $e$.
 
\begin{theorem}
For every GTN $\langle E, \mathrm{g}, \tau\rangle$, there is a corresponding model in CSP\#.
\end{theorem} 
 
\begin{proof} (Sketch) 
By the construction of the CSP\# process $e()$ and according to the definition of
the GTN $e$, all valid transition sequences of the CSP\#
process $e()$ correspond to valid plans of the GTN planning problem $P$.
\qed
\end{proof}

\paragraph{Example} Take the overall control of the AUV survey mission as an
example. At this granularity, the GTN is responsible for making high-level
decisions regarding the mission, such as which survey area to visit, in which
order, and how to visit it (enter from which direction and exit from which
direction). Assuming all the predefined locations are stored in an array, a
primitive task at this level is \emph{goto(i)}, which moves the AUV to location
i:

\begin{lstlisting}[mathescape=true]
goto(i) = [visited[i] == 0]go.i {
    currentPosition[0] = position[i][0]; 
    currentPosition[1] = position[i][1]; 
    visited[i] = 1; 
} $\rightarrow$ Skip;
\end{lstlisting}

\noindent In the $goto(i)$ task, the vector $visited[]$ records the status of
each location. The precondition $visited[i] == 0$ ensures that each location is
visited only once. 
Since $goto(i)$ is a primitive task, it does not contain subtasks/subgoals,
therefore, its formulation only involves the guard condition and the transition.

The compound task $survey(i)$ dictates which locations to visit for survey area
$i$. This task does not have explicit state transitions, but instead performs
state transitions in its subtasks ($goto()$ tasks). Following the translation
template, $survey(i)$ is formulated as:

\begin{lstlisting}[mathescape=true]
survey(i) = (goto($i_0$) <> $\cdots$ <> goto($i_n$)); (survey(i) <> Skip);
\end{lstlisting}

\noindent where $i_0,\cdots, i_n$ are the indices of the locations in survey area
$i$.

Similarly, the survey mission is formulated as below:

\begin{lstlisting}[mathescape=true]
mission() = (survey(0) <> $\cdots$ <> survey(m)); (mission() <> Skip);
\end{lstlisting}

\noindent where $0,\cdots, m$ are the indices for survey areas.

The overall GTN involves initialising the start position of the AUV, performing
the survey mission, and returning to the final position for recovery. This is
modelled as below where we omit the code of $initialise()$:

\begin{lstlisting}[mathescape=true]
rendezvous() = goto(finalPosition);
main() = (initialise() <> mission() <> rendezvous()); (main() <> Skip);
\end{lstlisting}

\noindent However, since the motivating example specifies that the three sub-GTNs of $main()$
should be executed sequentially, the above definition can be optimised as
below:

\begin{lstlisting}[mathescape=true]
main() = initialise(); mission(); rendezvous();
\end{lstlisting}

\subsection{Planning Under Resource Constraints}
\label{subsec:resources}


For most planning applications, considering resource constraints,
such as limited amount of available energy, is critical to the
quality and relevance of the produced plan. This is particularly true
in the application domain we consider as strategic
commanders aim at launching AUVs that are meant to operate
autonomously for extended period of time with limited resources.
Therefore, it is essential that these resource constraints are
correctly modelled in order to be able to produce plans that can be
fully realised, i.e., plans that do not require more resources
than available. Also, as unexpected events may arise during
the execution of plans, it is necessary to formulate plans that
minimise resource consumption in order to maximise the AUV's
resilience.

Suppose we wish to consider a finite set of $m$ resources $R = \{r_0,
\cdots, r_{m-1}\}$ and certain tasks that may consume or produce a
finite and discrete amount of one or several of these resources. To
do so we introduce, to the GTN modelling of a planning problem, a set
of $m$ new variables $\mathcal{V}_R = \{v_{r_0}, \cdots,
v_{r_{m-1}}\}$ modelling the amount of available resources. In the
initial state, the values of these variables correspond to the amount
of resources available on launch. When a tasks $e$ consumes
one or several resources, its guard $\mathrm{g}_e$ is extended so
that it can only be executed if the resources needed to perform it
are available before it executes. Additionally, its state transition
function $\tau_e$ is also extended so that it decreases the values of
the resource variables in order to reflect the resources consumed.
Similarly, when a tasks $e$ produces one of several resources, its
state transition function $\tau_e$ is extended so that it increases
the values of the resource variables in order to reflect the
resources produced.

\paragraph{Example} In the motivating example, we wish to model the AUV energy
consumption while moving based on the distance to travel. To do so we introduce
the variable \textit{energyLevel} which models the amount of energy left in the
battery as well as a function \textit{dist} that returns the distance of the
trajectory between two positions and the constant
\textit{energyRequiredByMeter} which is used to scale the energy consumption
linearly with respect to a travelled distance. We then modify the $goto(i)$
implementation as follows:

\begin{lstlisting}[mathescape=true]
goto(i) = [visited[i] == 0 && 
energyLevel >= dist(currentPosition, position[i])] go.i {    
    energyConsumed = dist(currentPosition, position[i]) * energyRequiredByMeter;
    energyLevel $-$= energyConsumed;
    currentPosition[0] = position[i][0]; 
    currentPosition[1] = position[i][1];
    visited[i] = 1; 
} $\rightarrow$ Skip;
\end{lstlisting}

These changes allow the states of the GTN modelling of a planning
problem to encompass available resource quantities and guarantee that
valid plans do not, at any time, consume more resources than
available. Furthermore, these changes also enable us to minimise
resource consumption by maximising the available resource
quantities. However, as several resources may be considered, this
leads to a multi-objectives optimisation problem that is
unfortunately not readily supported by PAT.

We solve this problem by modelling the connections between resources. Note that some resources might be more valuable than others with
respect to the mission objectives. Therefore, to avoid the need for
multi-objectives optimisation capability, we propose to reduce the problem to a
single-objective optimisation. To do so, we suggest the use of an extra
variable $\Lambda$ acting as a \emph{common currency} which is used, among other
things, to evaluate the overall state of resources. To update the value of
$\Lambda$ we require, for each resource $r \in R$, a \textit{conversion function}
$\lambda_r : \N \to \N$ relating the basic unit of a resource as modelled by variable
$v_{r}$ to the basic unit of value of $\Lambda$. Conversion functions used in practice include linear functions, logistic functions as well as exponential and logarithmic functions depending on the nature of the resources. 
Using these conversion functions, we further extend the state transition functions of tasks producing (respectively
consuming) resources so that they increase (respectively decrease) the value of $\Lambda$ accordingly. 
An important aspect of this approach is that it enables the comparison of any two sets of quantified resources by transitivity. As a result, maximising the value of $\Lambda$ minimises the overall resources consumption while accounting for the relative importance of the considered resources. Another important aspect of this approach is that it provides mission operatives with an economic perspective on the complex relations that govern the relative importance of available resources -- a familiar perspective people can relate to in everyday life. 

To illustrate the use of a conversion function we integrate the common currency into
the motivating example by inserting the following line after line 3 of the above code modelling the movement of the AUV, where $r_{energy}$ and $e_{energy}$ are user-defined constant:

\begin{lstlisting}[mathescape=true]
$\Lambda$ $-$= $\lfloor r_{energy}$ * $energyConsumed^{e_{energy}} \rfloor$;
\end{lstlisting}

Continuing with the AUV survey mission example, our model described above already takes the energy cost into account. To find a plan for the modelled GTN with respect to the energy cost, we first need to define
the condition for the overall goal:

\begin{lstlisting}[mathescape=true]
#define goal ($\forall i.$ visited[i] == 1) && (currentPosition[0] == finalPosition[0] && currentPosition[1] == finalPosition[1]);
\end{lstlisting}

\noindent which states that all the locations are visited, and the AUV's
current location is the final position. We then use PAT to find a plan that
yields minimal energy cost by model checking the following assertion:

\begin{lstlisting}[mathescape=true]
#assert main() reaches goal with max($\Lambda$);
\end{lstlisting}    

\subsection{Goal Reasoning}
\label{subsec:goal}

In this section we further discuss the concepts that enable our model checking based approach to deal with run-time goal reasoning.

\subsubsection{Reasoning About Rewards/Penalties of Goals} 

Due to environment constraints and resource constraints, the completion
of one or several goals may not be possible, or perhaps not worthwhile.
Further, goals may not have the same priority. Some goals may be more important
to the success of the mission than others. Additionally, as one of the
underlying directives is to minimise resources consumption, the produced plans
may not consider secondary objectives and only fulfil the minimum requirements
in order to complete the mission if the incentive to do so is not correctly
modelled.

To cope with these challenges, we propose to associate the achievement of a goal
with a reward function relating the goal completion to an amount of the basic unit of value of $\Lambda$ --
the previously introduced variable acting as the common currency. 
In this setting, maximising the value of $\Lambda$ prioritises and incentivises
the completion of goals providing the most rewards while compromising with the resources they require to be completed. 
Further, as the resources conversion functions and the reward functions can be arbitrarily complex arithmetic functions, this provides a way to assess trade-offs between complex, competing criteria for a large number of resources and goals. 

These economic notions therefore lead to the formulation of highly cost-effective plan. Additionally, when multi-agents missions are considered, they provide further benefits as market-based mechanisms~\cite{clearwater1996market} can be leveraged to obtain greater collaboration among agents as well as to optimise resources and tasks allocation. 
These mechanisms also provide non-technical operatives the means to leverage their day to day economic knowledge to specify technical details of the missions that have to be accomplished by the agent.

\paragraph{Example} Returning to the motivating example, we wish to prioritise
the recovery of the vehicle ($rendezvous()$) over the completion of the survey
($mission()$). To achieve this, we first insert the following code into
$goto(i)$ (between the curly braces):

\begin{lstlisting}[mathescape=true]
$\Lambda$ $+$= reward$_{survey}$;
\end{lstlisting}

\noindent We then modify the definition of $rendezvous()$:

\begin{lstlisting}[mathescape=true]
rendezvous() = rend{$\Lambda$ $+$= reward$_{rendezvous}$;} $->$ goto(finalPosition); 
\end{lstlisting}
    
\noindent We set $reward_{rendezvous}$ to be far greater than $reward_{survey}
\times N$ where $N$ is the total number of positions in the model. We also have
to ensure that $reward_{survey}$ is greater than $\lfloor r_{energy}$ * $energyConsumed^{e_{energy}} \rfloor$, otherwise PAT will choose not to visit any
position at all. Finally, we modify the $goal$ so that visiting all positions
and returning to recovery position is no longer mandatory. Rather, we use a
more flexible goal, defined as below:

\begin{lstlisting}[mathescape=true]
#define goal $\forall i \in C.$ visited[i] == 1;
\end{lstlisting}

\noindent where $C$ is a subset of positions that are critical and will
override the optimisation on reward/penalty. Now when we model check 

\begin{lstlisting}[mathescape=true]
#assert main() reaches goal with max($\Lambda$);
\end{lstlisting}

\noindent If the $energyLevel$ is sufficient to visit all positions and go to
the recover position, then PAT will output such a plan with minimal energy
consumption. Otherwise, if the $energyLevel$ is insufficient due to unexpected
events such as strong current, energy spent on detour or surveying uncertain
objects, etc., PAT will try to find a plan that ensures that the positions in
$C$ are visited, and that $rendezvous()$ is far more likely to be executed than
visiting a few more positions.

\subsubsection{Reasoning About Consistency of Goals}
Consider the following scenario: the AUV has finished the survey
mission and now has to report the results. There are two ways to
report: (1) acoustic communication with a nearby friendly surface
vessel; (2) surface and use satellite communication. Suppose there is
no friendly surface vessel nearby, then the AUV will choose the
second method. However, suppose there is a hostile surface
vessel, which the AUV should avoid. Now the AUV has two goals: report
using satellite communication and avoid the hostile surface vessel.
The underlying plans for these two goals have conflicts, and the two
goals should not be pursued at the same time. 


Since PAT can determine whether a condition is satisfiable or not in execution,
we can also use PAT to determine the satisfiability of the conjunction of
several conditions. To solve the above issue, we first formulate the goals as
the conditions below:

\begin{lstlisting}[mathescape=true]
#define goalCompleteSurvey auvCom == 1; 
#define successfulSurvey goalCompleteSurvey && hvContact == 0;
\end{lstlisting}

\noindent The first goal says that the AUV has done the communication, the
second goal is a compound goal that consists of the first goal and
that the AUV does not surface when the hostile vessel is nearby
($hvContact == 0$). We define a task $auvReport$ which
consists of subtasks $auvAcousticCom$ and
$auvSurfaceCom$, which represents communication with friendly
surface vessel and with a satellite respectively. 

\begin{lstlisting}[mathescape=true]
auvAcousticCom() = [fvInRange]comFV{auvCom = 1;} $->$ auvReport();

auvSurfaceCom() = [!fvInRange]comS{auvDepth = 0; energyLevel -= 10; auvCom = 1; if (hostileInRange) hvContact = 1;} $->$ auvReport();

auvReport() = auvAcousticCom() [] auvSurfaceCom();
\end{lstlisting}

\noindent The condition $fvInRange$ checks whether the friendly vessel is in
range for acoustic communication. Verifying the below assertion, which states
that $auvReport$ can reach a state where the communication has been done and
the AUV has not had contact with the hostile vessel, would return negative
by PAT.

\begin{lstlisting}
#assert auvReport() reaches successfulSurvey;
\end{lstlisting}

\noindent This means that the above two goals are incompatible, and
PAT cannot find an execution path to satisfy both. To resolve this
issue, we can add a new task that moves the AUV away from the hostile
vessel, as coded below:

\begin{lstlisting}
auvAvoidContact() = case {
        hostileInRange: auvMove();auvReport()
        default: auvReport()
};
auvReport() = auvAcousticCom() [] auvSurfaceCom() [] auvAvoidContact();
\end{lstlisting}

\noindent Now PAT returns affirmative for the above verification and gives a
plan to achieve the goal $successfulSurvey$.

We can extend this solution to check incompatibility of a set of
goals. Given a set $S$ of goals, we can use PAT as a black-box and implement Algorithm~\ref{alg:muc}~\cite{mucslides} to find the minimal
set of goals that are incompatible. We can also find the set of achievable
goals, and update the model to resolve unachievable goals if necessary. Algorithm~\ref{alg:muc} is an elementary method for efficiently finding the minimal unsatisfiable core of a set of formulae by divide and conquer.

\begin{algorithm}[t!]
    \caption{A simple algorithm for finding minimal unsatisfiable core (MUC). To find a MUC of $S$, call $Minimise (S,\emptyset)$.}
    \label{alg:muc}
    \begin{algorithmic}
        \Procedure{Minimise}{$S$, $S_0$}
        \State Randomly partition $S$ into two sets $S'$ and $S''$ of the same size.
            \If{$S' \land S_0$ is unsatisfiable}
            \State \Return $Minimise(S', S_0)$;
            \ElsIf{$S'' \land S_0$ is unsatisfiable}
            \State \Return $Minimise(S'', S_0)$;
            \Else \Comment{$S' \land S_0$ and $S'' \land S_0$ are both satisfiable}
            \State $S'_{min} \gets Minimise(S', S_0 \land S'')$;
            \State $S''_{min} \gets Minimise(S'', S_0 \land S'_{min})$;
            \State \Return $S'_{min} \land S''_{min}$;
            \EndIf
        \EndProcedure
    \end{algorithmic}
\end{algorithm}

\subsection{Performance Testing} 
\label{subsec:perf-test}


To judge the feasibility and scalability of
the model checking based approach, we have tested two levels of planning
details. (i) The first level consists of finding an order of the areas to
survey so that it minimises the energy cost of the mission. At this level we
abstract away the entry, the internal path and the exit point of each survey
area. The second level (ii) enables the entry and exit point of each survey
area to be determined. These levels respectively
correspond to two GTNs of increasing complexity.

We ran the testing on the NVIDIA Jetson TX2 -- a
power-efficient embedded chip that is equipped in a customised REMUS-100 underwater
vehicle at Defence Science and Technology (DST) Australia. We report the results in Table~\ref{tab:testing_tx2}, in which each configuration is run 5 times and the average of the CPU time and memory usage are displayed. One could theoretically also model the ``lawn mowing'' path inside each survey area, but it is more of an actuation problem than a planning problem, thus we do not test it here.

\begin{table}[t!]
\centering
\begin{tabular}{|l|l|l|l|}
\hline
Level & \# of Survey Areas & avg. CPU Time (s) & avg. Memory Usage (MB)\\
\hline
\multirow{4}{*}{1} & 2 & 0.005 & 8.4\\
& 3 & 0.01 & 8.4\\
& 4 & 0.03 & 8.4\\
& 5 & 0.16 & 11.7\\
\hline
\multirow{4}{*}{2} & 2 & 0.14 & 11.2\\
& 3 & 3.20 & 66.9\\
& 4 & 40.26 & 383.8\\
& 5 & 290.70 & 796.3\\
\hline
\end{tabular}
\caption{Performance testing for planning and goal reasoning in two levels of PAT models. Level 1 decides which survey areas to visit and the order to visit them. Level 2 further decides the entry and exit points of each survey area.}
\label{tab:testing_tx2}
\vspace{-10px}
\end{table}



The model complexity has a significant impact on the run-time and memory usage of the goal reasoning and planning phase. This is not surprising and is mainly due to the explosion of the state-space size -- an issue commonly encountered by model checkers~\cite{valmari1996state}. On the other hand, the REMUS-100 AUV only has a cruising speed of 5.4 km/h, which means that the software has plenty of time to perform re-planning during the mission. The other targeted hardware, the Ocean Glider, is even slower since it relies on water movement to generate forward thrust. We conclude that Level 1 is feasible, and Level 2 is feasible only when the number of survey areas is less than 3. Note that both these levels are high-level operations. We still need to convert high-level operations to low-level operations which can be actuated by the hardware.

The above results highlight the trade-off between performance and guarantees. An approach based solely on model checking is at the moment intractable whereas an approach based solely on heuristics do not provide sufficient guarantees about missions critical elements. Therefore, the above empiric results support the design choice of a hybrid approach for goal reasoning and planning. That is, PAT is suitable for making critical high-level decisions, whereas we need to rely on an external program to translate the high-level plans into low-level plans. The verification of this translation is non-trivial: it includes details such as showing that turning the rudder of the AUV at a certain degree corresponds to going a certain direction in the high-level plan. Such details are hardware-dependent and are not in the scope of this paper.
Nonetheless, a carefully designed GTN at an appropriate level of details can, in the context of a hybrid approach, provide better trustworthiness and reliability for the high-level decision-making.
\section{GRAVITAS: A Trustworthy Framework for Planning and Goal Reasoning}
\label{sec:GRAVITAS}


This section describes the \textit{Goal Reasoning 
And Verification for Independent Trusted Autonomous 
Systems} (GRAVITAS) -- an automated system which enables autonomous
agents to operate with trustworthy high-level plans in a dynamic environment.

\subsection{An Overview of GRAVITAS}
\label{subsec:GRAVITAS_overview}

The GRAVITAS framework follows a cyclic pattern composed of 
four main phases: Monitor, Interpret, Evaluate and Control, which are illustrated in 
Figure~\ref{fig:tas_overview}.

\begin{figure}[ht!]
\centering
\includegraphics[width = 0.9\textwidth]{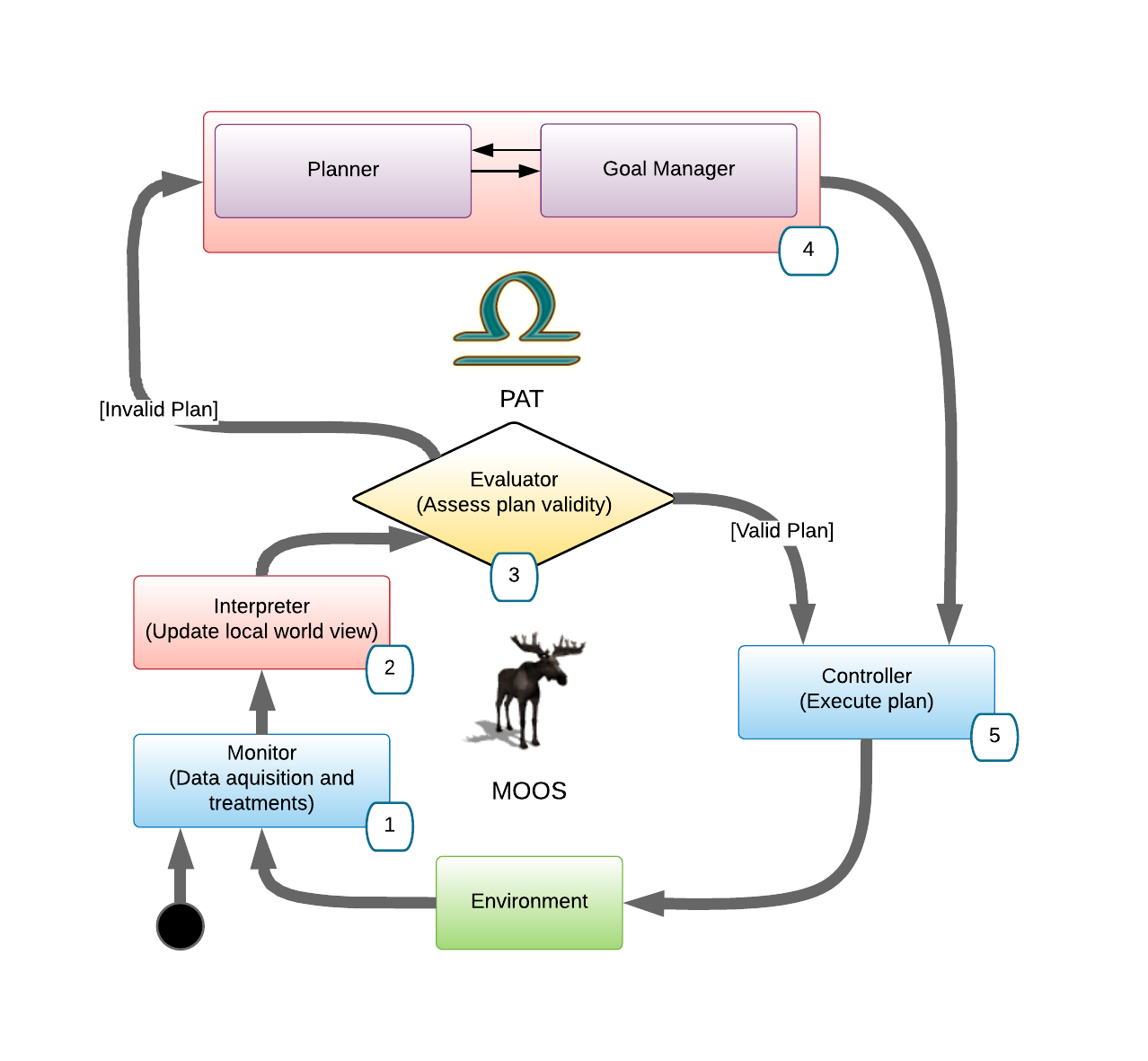}
\caption{Overall workflow of GRAVITAS.}
\label{fig:tas_overview}
\end{figure}

The main operative cycle of GRAVITAS begins with the Monitor (1). This
component perceives the environment through the 
signal processing and fusion of the raw outputs of available sensors. 
For AUVs, examples of sensors includes
accelerometers, gyroscopes, pressure sensors and GPS.
It is also in charge of processing these data in order to provide
information such as the estimated position and the speed of the agent to the 
Interpreter (2). This step notably involves techniques such as target tracking
which we will not detail here~\cite{challa*b}. Once the Interpreter
(2) receives the required information, it updates the agent's local
model of the system and its environment. This formally defined local
model is then forwarded to the Evaluator (3) -- a component in charge
of assessing the validity of the previously established plan with
respect to pre-defined specifications. If the Evaluator assesses the plan as valid, 
the Controller (5) is tasked with executing the plan. Otherwise,
if the Evaluator (3) finds the plan invalid
e.g., an uncertain event creates inconsistencies in the
previously established plan and the mission requirements, a new
plan needs to be formulated. The formulation of a new plan is
accomplished by the joint operation of the Planner and Goals
Manager components (4). After a new plan is formulated, the Controller (5) is tasked 
with executing this plan. This step involves processing based on control
theory~\cite{lee1967foundations} which we do not discuss here.

The components in the lower loop in Figure~\ref{fig:tas_overview} are
orchestrated via the Mission Oriented Operating Suite~\cite{newman2008moos}
(MOOS) -- a middleware mainly in charge of the communication. The main
computational workload of the Evaluator (3), the Planner and the Goal Manager (4)
components are powered by PAT. Note that although conceptually the planner and
the goal manager are two separated components, in our implementation they are
realised in the same PAT model, as discussed in the examples throughout Section~\ref{sec:plan}.
Also, to achieve high efficiency in real-life applications, we use a hybrid approach discussed in Section~\ref{subsec:perf-test}.

\subsection{Verification of Planning and Goal Reasoning Models}
\label{subsec:veri_model}

The key advantage of the model checking based approach is that we can formally verify certain properties for the planning and goal reasoning model. This verification guarantees that the model only permits ``correct'' high-level plans. Since the verified model is directly used to generate high-level plans in the planning and goal reasoning phase, we can ensure that the generated high-level plans not only are optimised by for max rewards (resp. min penalties), but also are ``correct'' with respect to the verified properties.

The verification itself is straightforward since the model is already in CSP\#. We only need to formulate the properties in the specification language (cf. Section~\ref{sec:pre}) and use model checking to verify them. 

\paragraph{Example} In the AUV survey example, we are interested in checking
whether the model would permit an execution sequence in which the AUV hits an
obstacle. The below Boolean condition expresses that the position of AUV does
not overlap with any position of obstacles.

\begin{lstlisting}
#define dontRunIntoObstacle (&& index:{0..iNumberOfObstacles-1}@(obstacles[index][0] != auvPosition[0] || obstacles[index][1] != auvPosition[1]));
\end{lstlisting}

Using LTL, we can check whether this condition holds \emph{for all}
subsequent states in the execution. This is realised by an assertion
of the form

\begin{center}
$p \vdash \Box c$
\end{center}

\noindent where $p$ is a process in CSP\#, $c$ is the condition we
need to check, and $\Box$ is a modality in LTL that means $c$ holds
for all subsequent states. The above verification is realised in the
code below 
\begin{lstlisting}
#assert main() |= [] dontRunIntoObstacle;
\end{lstlisting}
and PAT can automatically return ``yes'' as the result. Thus we obtain the following lemma:

\begin{lemma}
The example planning and goal reasoning model described in Section~\ref{sec:plan} does not generate plans where the AUV runs into any obstacle.
\end{lemma}

Using the same technique, we have verified the following lemmas:

\begin{lemma}
The example model described in Section~\ref{sec:plan} does not generate plans where the AUV runs out of battery during the mission.
\end{lemma}

\begin{lemma}
The example model described in Section~\ref{sec:plan} does not generate plans where the AUV surfaces at a location within 3 units of distance of a hostile vessel.
\end{lemma}

\subsection{Interacting with Un-trusted Components}
\label{subsec:interact-untrusted}

Although the Level 2 planning and goal reasoning model in Table~\ref{tab:testing_tx2} suffices in our demonstration of the AUV survey mission, there might be other applications where model checking cannot provide detailed plans in time. For instance, the user may need to adopt heuristic-based planning techniques for UAVs and land vehicles because they run faster.  

Inspired by Clarke et al.'s counterexample-guided abstraction refinement~\cite{Edmund2000}, we propose to integrate heuristic-based planning techniques as an ``un-trusted component'' as follows: We treat the heuristic method as an high-level plan generator. Whenever the heuristic method generates a plan, we simulate this plan using the corresponding high-level planning and goal reasoning model, i.e., the \emph{CSP\# model}, in PAT. This simulation is much faster than model checking because we only need to check one path of actions instead of checking all paths. If the simulation is successful, then this plan is in the set of plans that can be generated by the CSP\# model. If the CSP\# model has been verified as described in Section~\ref{subsec:veri_model}, then this plan is correct with respect to the verified properties. If the simulation fails, then we add the old plan into a set of \emph{disabled plans} and constraints the heuristic method such that it does not generate one of the disabled plans. This procedure provides plans that have the same formal guarantee as those generated by PAT, but this procedure may not yield optimal plans. Nonetheless, this procedure provides the means to interact with existing heuristic-based planning techniques generally employed without safe-guards in a reliable way.  


\section{Implementation of GRAVITAS} 

In situ experimentation is very
expensive and slow. While it is mandatory to the final evaluation of the
implementation, in this paper we focus on assessing and demonstrating the
feasibility of the proposed goal reasoning and planning approach in a virtual
environment. Notable challenges include controlling the complexity of the GTN
so that the embedded hardware of the AUV is able to carry the computational
load in a reasonable time (i.e., less than a minute) and the transposition of a
discrete plan as issued by PAT into its continuous counterpart so that it can
be enacted by the AUV. 


We have implemented the proposed approach and integrated
it within a virtual environment closely simulating the mission described in
Section~\ref{sec:example}. We first introduce the integration of PAT within a
\emph{community} of MOOS applications~\cite{newman2008moos}. We then report the
obtained results and discuss the conclusion drawn from them.

The experimental setup is composed of a MOOS application community (referred to as the ``community'' in the sequel) that corresponds to the AUV internal software environment as well as a set of applications that aim at simulating the external environment. This community includes the following: 
\begin{description}
	\item[MOOSDB:] All communication happens via this central server application.
	\item[uSimMarine:] A 3D vehicle simulator that updates vehicle state, position and trajectory, based on the present actuator values and prior vehicle state.
	\item[pMarinePID:] A PID controller for heading, speed and depth.
	\item[pMarineViewer:] A GUI rendering vehicles and associated information during operation or simulation.
	\item[pSideScanner:] A simulator that reports objects identified by the side scanners of the simulated AUV. 
	\item[pPATApp:] The application that integrate PAT and provide goal reasoning and planning ability to the AUV.
\end{description}

\textit{pPATApp} implements the GRAVITAS framework as described in
Section~\ref{sec:GRAVITAS}. It subscribes to and monitors channels which
broadcast information about the general state of the AUV (e.g., position,
speed, heading) as well as information about the objects detected by the side
scanners. Then, at each iteration of its internal loop, it interprets this
information and models a local world view of the environment. Based on this
internal representation and according to the proposed planning approach, it
evaluates the actual plan being enacted and, if required, updates it before
enacting it by publishing the desired heading, speed and depth of the AUV to
the community.

The plan issued by PAT as a part of the re-planning step is a discrete sequence
of primitive tasks (e.g. go to 3D position) that require some processing in order
to be enacted by the AUV as actuators commands (e.g., set heading, set speed).
For instance, the trajectory between several way-points set by the plan has to
be compliant with the maximum turn-rate of the AUV. To solve this issue, as a
proof of concept we implemented an algorithm based on piecewise Bezier curves
composition with continuous curvature constraint for continuous path
planning~\cite{choi2010piecewise}. In the future we plan on using a more
advanced low-level planning approach such as the FMT* algorithm~\cite{janson2015fast}
that will enable us to consider trajectories based on 3D current dynamics as
well as uncertainty in the AUV position.

\section{Simulation in MOOS pMarineViewer} 
\label{sec:exp}

We demonstrate a case study
scenario in a simulation in MOOS. In this scenario, we intend to capture
GRAVITAS's capabilities in dealing with dynamic events during execution. We
create a survey mission similar to Figure~\ref{fig:overall_mission}. Note that
although the following screenshots are in 2D, the simulation is actually in a
3D environment.

\begin{figure}[ht!]
	\begin{subfigure}{.5\textwidth}
	  \centering
	  \includegraphics[width=.98\linewidth]{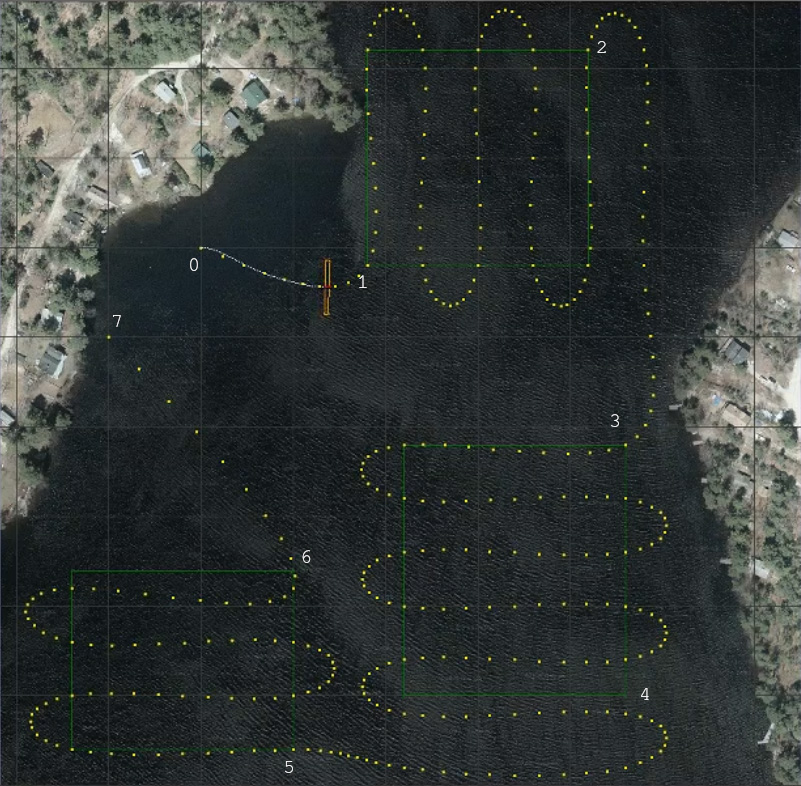}
	  \caption{Initial plan.}
	  \label{fig:survey1}
	\end{subfigure}
	\begin{subfigure}{.5\textwidth}
	  \centering
	  \includegraphics[width=.98\linewidth]{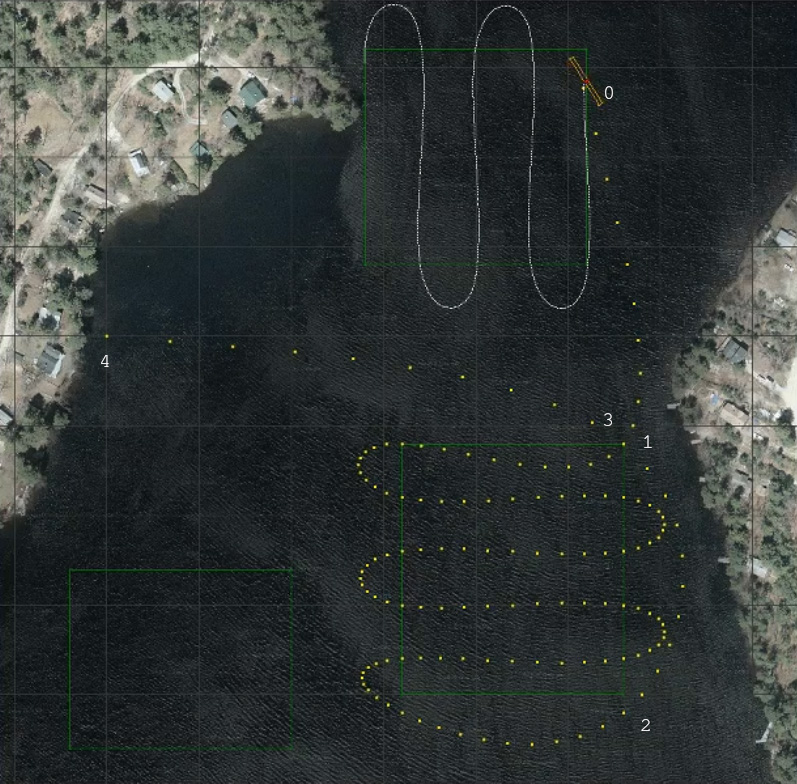}
	  \caption{1st re-planning.}
	  \label{fig:survey2}
	\end{subfigure}
	\begin{subfigure}{.5\textwidth}
		\centering
		\includegraphics[width=.98\linewidth]{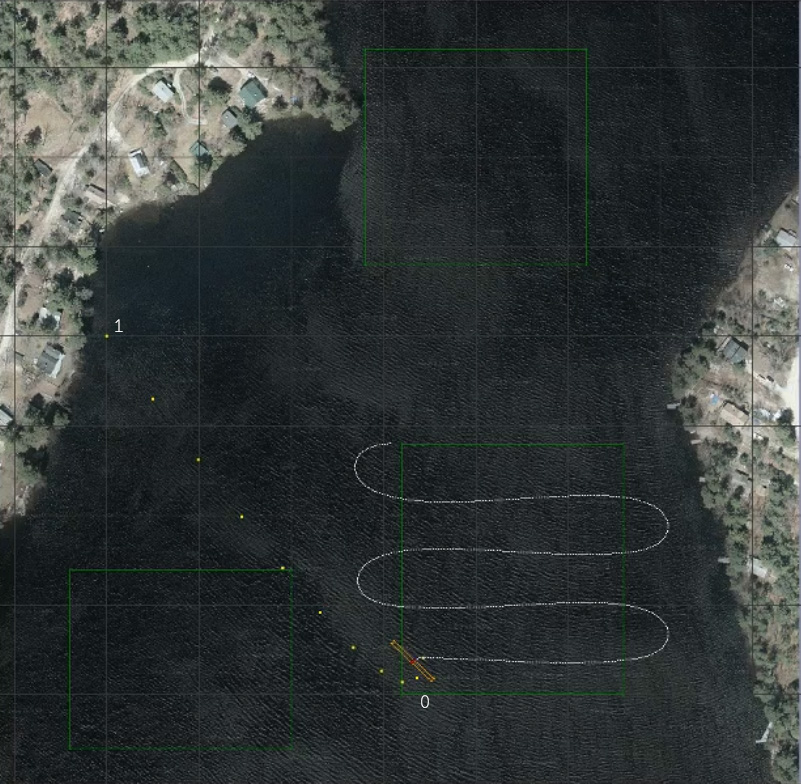}
		\caption{2nd re-planning.}
		\label{fig:survey3}
	\end{subfigure}
	\begin{subfigure}{.5\textwidth}
		\centering
		\includegraphics[width=.98\linewidth]{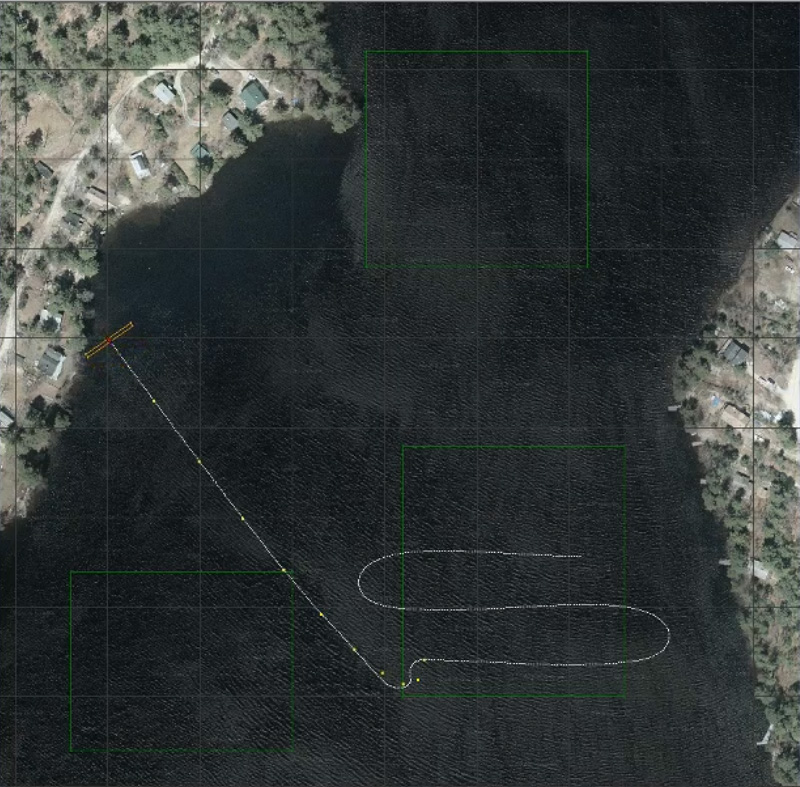}
		\caption{Mission finished successfully.}
		\label{fig:survey4}
	\end{subfigure}
	\caption{Screenshots of a survey mission simulated in MOOS pMarineViewer.}
	\label{fig:fig}
\end{figure}

We set 3 survey areas: lower-left (LL), upper-right (UR), and lower-right (LR),
with rewards 22807, 51918, 31313 respectively. Initially, the AUV has an energy
level of 60000. During execution, we randomly generate a strong water current,
with a chance of 20\%, that doubles the energy consumption for an uncertain
period of time. For simplicity in this example, we trigger goal reasoning and re-planning at
the end of each survey area, although these can be done as frequently as required provided that the computation does not take longer than the interval between re-plannings.

Figure~\ref{fig:survey1} shows the initial plan computed by GRAVITAS. The
numbers indicate the high-level plan computed by PAT and the dots indicate the
low-level plan generated by the actuator. That is, PAT finds the optimal order
as well as the entry and exit points for the survey areas, and the actuator
computes a smooth path that the AUV can follow. The ``wing'' of the AUV
indicates the coverage of the side scan sonar.


In this run, the random generator creates a strong current during the first
survey. As a
result, the expected energy consumption for the first survey is 14400, but the
actual consumption is 23284. Consequently, the Interpreter in GRAVITAS uses a
simple ``learning'' to update the expected energy consumption for future
execution to 162\% of the estimation. However, this unexpected change causes
re-planning in which PAT decides that there is insufficient energy to complete
3 survey areas, and then finds a new plan to optimise the outcome, as shown in
Figure~\ref{fig:survey2}. In the new plan, PAT chooses to only survey area LR
because it yields more reward.


During the transit to the second survey, the water current has returned to
normal. At the end of the second survey, the Controller in GRAVITAS discovers
that survey area LR has been fully covered before going through the last pass
in the ``lawn-mowing'' pattern. Therefore it triggers re-planning and PAT and
the actuator enact a new plan, shown in Figure~\ref{fig:survey3}, to directly
go to the rendezvous point. At the same time, the Interpreter captures that the
expected energy consumption for the second survey is 27216, but the actual
consumption is 17411, so it lowers the scale of future energy consumption to
104\% of the estimation. There is no more strong current on the way to the
rendezvous point, and the expected energy consumption is roughly the same as
the actual value, and the AUV successfully finishes the mission, as shown in
Figure~\ref{fig:survey4}. In this case study, the (re-)planning takes around 1
second, which is fast enough for the operation of AUV.


\section{Related Work}
\label{sec:related}



Different approaches exist according to the
assumptions about the domain, the goals, the plans and the planning
algorithm. Conceptually, the domain evolves according to the
performed actions, a controller provides the actions according to the
observations on the domain and a plan~\cite{Ghallab2016}. 
An example
of applying automated reasoning techniques on planning is Kress-Gazit
et al.'s framework which automatically translates high-level tasks
defined in linear temporal logic formulae to hybrid
controllers~\cite{KressGazit2009}. This framework allows for reactive
tasks, which may change depending on the information the robot
gathers at runtime. This is similar to the goal reasoning literature
where goals may change depending on the environment at runtime.

This work follows 
the ``planning as model checking'' paradigm,
which
dates back to 1990s, e.g., in the work by Giunchiglia and
Traverso~\cite{giunchiglia1999planning}. They proposed to solve
(classical) planning problems model-theoretically, where planning
domains are formalised as semantic models, properties of planning
domains are formalised as temporal formulae, and planning is done by
verifying whether temporal formulae are true in a semantic model.
This idea has been studied and improved in their subsequent
work~\cite{cimatti2000conformant,cimatti1997planning,bertoli2001heuristic},
which involves using Binary Decision Diagram based
heuristic symbolic search. Similar ideas have been used in planners
such as MIPS~\cite{edelkamp2001mips}, which can effectively handle
the STRIPS subset of the PDDL, and some additional features in ADL.


Closely related to the above work is the verification and validation (V\&V)
based method of Bensalem et al.~\cite{bensalem2014verification}. They argue
that constructing correct and reliable planning systems is error-prone due to
the non-deterministic nature of planning problems, thus it is important to
develop V\&V methods for planners to ensure that the generated plans are
correct. To achieve this, the authors proposed to \emph{use} V\&V techniques to
perform planning, and use planning to perform V\&V. This work is similar in the
sense that we are using model checking techniques to perform planning, and
since the planning system is built upon the model checker, we can also verify
correctness and safety issues of the plans and goals. As a result, we can not
only output plans that are efficient in certain criteria, but also those that
are verified safe and correct, which is essential in building trusted
intelligent agent and is often required in mission-critical operations.



Goal reasoning has been used in a number of projects about
controlling autonomous machines in a dynamic environment. 
Many goal reasoning systems follow a \emph{note-assess-guide} procedure,
and extend it with a cycle of executions to handle the dynamics of the environment and perform goal reasoning and re-planing on-the-fly.
Cox et
al.~\cite{Cox2016} propose to use classical planning to formalise
goal reasoning. They present an architecture with a cognitive layer
and a metacognitive layer to model problem-solving and dynamic event
management in self-regulated autonomy. The architecture is realised
in the Metacognitive Integrated Dual-Cycle Architecture (MIDCA) version 1.3,
which is shown useful in experiment. A detailed account is given by
Dannenhauer~\cite{dannenhauer2017self}. 

Roberts et al.~\cite{RobertsAJAWA15} give more detailed definitions
of goal reasoning in their framework. They divide the states and
goals into two parts: the external part is a modified or incomplete
version of the transition system, and the internal part represents
the predicates and state required for the refinement strategies. The
authors use a data structure called \emph{goal memory} to represent
the relationship between goals, subgoals, parent goals etc., and
propose to solve the goal reasoning problem using refinement. They
use a \emph{goal lifecycle} model to capture the evolution of goals
and the decision points involved in the process. The goal lifecycle
includes the formulation, selection, expansion, execution, dispatch,
evaluation, termination, and discard of goals. This model is adapted
by Johnson et al.~\cite{Johnson2016}, who give a system called Goal
Reasoning with Information Measures. In the scenario of controlling
Unmanned Air Vehicles to survey certain areas, the goals are
formulated with parameters such as \emph{maximum uncertainty} in the
search area, \emph{acceptable uncertainty} under which the goal is
considered complete, and \emph{deadline} by which the search must
complete. The goal reasoning method is shown useful for unmanned
aerial vehicles operating in dynamic environments.

A more theoretical foundation about planning and goal reasoning is
surveyed by Alford et al.~\cite{ASR16}. The authors unify HGN
planning and HTN planning into GTN planning. They also provide
plan-preserving translations from GTN problems to HTN semantics.
Several computability and tractability results are given. For
example, GTN, HTN, and HGN are semi-decidable, and a restricted form
called $\text{GTN}_I$ is NEXPTIME. An application of HTN planning realised by symbolic model checking is presented by Kuter et al.~\cite{Kuter2005AHT}. While their work is focused on the theoretical foundation of the problem and they assume full-observability, this paper is more concerned with a more concrete real-life problem: the execution of the AUV in an uncertain environment. Thus this paper is more focused on practical issues that arise when solving the AUV survey problem.


One interesting use case of goal reasoning is goal selection.
Rabideau et al.~\cite{RabideauCM11} give a tractable goal selection
method algorithm specialised for selecting goals at runtime for
re-planning in a system where computational resources are limited and
the complete goal set oversubscribe available resources. Kondrakunta
and Cox~\cite{Kondrakunta2017} also consider the situation where an
agent has more goals than can complete in a given time constraint and
show how an intelligent agent can estimate the trade-off between
performance gains and resource costs. Another important aspect of
goal reasoning is to detect inconsistency or incompatibility of goals
and plans. Tinnemeier et al.~\cite{Tinnemeier2008} propose a
mechanism to process incompatible goals which have conflicting plans.
They argue that the agent should not pursue goals with conflicting
plans, and their mechanism can help the agent choose from
incompatible goals.

An important application of our project is applying the planning and
goal reasoning framework to AUVs. Among many relevant papers, goal reasoning for AUVs~\cite{Wilson2016} is particularly interesting. The authors use a
goal-driven autonomy conceptual model which has three parts: the
planner, the goal controller, and the state transition system. The
goal reasoning problem is formalised in PDDL, which is the standard
language for representing classical planning problems and is widely
used by many planners. The authors test their approach in simulations
where the AUV surveys a defined area and it has to respond (change
the goal) to the actions from a nearby unmanned surface vehicle
dynamically. Cashmore et al.'s work~\cite{Cashmore2018} describes a planning algorithm for AUVs. Like many other related papers, their work assumes certain requirements that are slightly different from our settings. For example, they are focused on temporal planning with time constraints whereas our mission does not have such constraints.

\section{Conclusion}
\label{sec:conc}



This work describes a decision-making framework named GRAVITAS for autonomous systems. The GRAVITAS framework not only provides theoretical foundation for hierarchical planning and goal reasoning, i.e., modelling GTNs using CSP\# and using model checking to perform planning and goal reasoning, but also includes practical implementations via the model checker PAT and the MOOS application community. This framework is ultimately realised on a hardware chip that runs on the REMUS-100 AUV.
Our simulation has shown that the
model checker PAT is sufficient to perform high-level decision making tasks.
We have also developed various auxiliary functionalities in GRAVITAS to extend the
high-level PAT plan into low-level plans for actuation. An important future work is to improve the level of trustworthiness by extending the verification from high-level plans to low-level plans. We are also planning to conduct more realistic simulations, and will attempt to show 
in situ ability of our approach in real-world demonstrations.

\bibliographystyle{spmpsci}
\bibliography{main.bib}

\begin{thebibliography}{10}
\providecommand{\url}[1]{{#1}}
\providecommand{\urlprefix}{URL }
\expandafter\ifx\csname urlstyle\endcsname\relax
  \providecommand{\doi}[1]{DOI~\discretionary{}{}{}#1}\else
  \providecommand{\doi}{DOI~\discretionary{}{}{}\begingroup
  \urlstyle{rm}\Url}\fi

\bibitem{ASR16}
Alford, R., Shivashankar, V., Roberts, M., Frank, J., Aha, D.W.: Hierarchical
  planning: Relating task and goal decomposition with task sharing.
\newblock In: IJCAI, pp. 3022--3029 (2016)

\bibitem{AOC06}
Anderson, M.L., Oates, T., Chong, W., Perlis, D.: The metacognitive loop {I:}
  enhancing reinforcement learning with metacognitive monitoring and control
  for improved perturbation tolerance.
\newblock J. Exp. Theor. Artif. Intell. \textbf{18}(3), 387--411 (2006)

\bibitem{bai2013authscan}
Bai, G., Lei, J., Meng, G., Venkatraman, S.S., Saxena, P., Sun, J., Liu, Y.,
  Dong, J.S.: Authscan: Automatic extraction of web authentication protocols
  from implementations.
\newblock In: NDSS (2013)

\bibitem{bensalem2014verification}
Bensalem, S., Havelund, K., Orlandini, A.: Verification and validation meet
  planning and scheduling (2014)

\bibitem{bertoli2001heuristic}
Bertoli, P., Cimatti, A., Roveri, M.: Heuristic search+ symbolic model
  checking= efficient conformant planning.
\newblock In: IJCAI, pp. 467--472. Citeseer (2001)

\bibitem{Cashmore2018}
Cashmore, M., Fox, M., Long, D., Magazzeni, D., Ridder, B.: Opportunistic
  planning in autonomous underwater missions.
\newblock IEEE Transactions on Automation Science and Engineering
  \textbf{15}(2), 519--530 (2018).
\newblock \doi{10.1109/TASE.2016.2636662}

\bibitem{challa*b}
Challa, S., Morelande, M.R., Mu$\check{\text{s}}$icki, D., Evans, R.J.:
  Fundamentals of Object Tracking.
\newblock Cambridge University Press (2011)

\bibitem{choi2010piecewise}
Choi, J.W., Curry, R., Elkaim, G.: Piecewise bezier curves path planning with
  continuous curvature constraint for autonomous driving.
\newblock In: Machine learning and systems engineering, pp. 31--45. Springer
  (2010)

\bibitem{cimatti1997planning}
Cimatti, A., Giunchiglia, E., Giunchiglia, F., Traverso, P.: Planning via model
  checking: A decision procedure for ar.
\newblock In: European Conference on Planning, pp. 130--142 (1997)

\bibitem{cimatti2000conformant}
Cimatti, A., Roveri, M.: Conformant planning via symbolic model checking.
\newblock J. Artif. Intell. Res.(JAIR) \textbf{13}, 305--338 (2000)

\bibitem{Edmund2000}
Clarke, E., Grumberg, O., Jha, S., Lu, Y., Veith, H.: Counterexample-guided
  abstraction refinement.
\newblock In: E.A. Emerson, A.P. Sistla (eds.) Computer Aided Verification, pp.
  154--169. Springer Berlin Heidelberg, Berlin, Heidelberg (2000)

\bibitem{clarke1986automatic}
Clarke, E.M., Emerson, E.A., Sistla, A.P.: Automatic verification of
  finite-state concurrent systems using temporal logic specifications.
\newblock ACM Transactions on Programming Languages and Systems (TOPLAS)
  \textbf{8}(2), 244--263 (1986)

\bibitem{Clarke2000}
Clarke Jr., E.M., Grumberg, O., Peled, D.A.: Model Checking.
\newblock MIT Press, Cambridge, MA, USA (1999)

\bibitem{clearwater1996market}
Clearwater, S.H.: Market-based control: A paradigm for distributed resource
  allocation.
\newblock World Scientific (1996)

\bibitem{Cox2016}
Cox, M.T., Alavi, Z., Dannenhauer, D., Eyorokon, V., Munoz-Avila, H., Perlis,
  D.: {MIDCA}: A metacognitive, integrated dual-cycle architecture for
  self-regulated autonomy.
\newblock In: AAAI, pp. 3712--3718 (2016)

\bibitem{dannenhauer2017self}
Dannenhauer, D.: Self monitoring goal driven autonomy agents.
\newblock Ph.D. thesis, Lehigh University (2017)

\bibitem{mucslides}
Dillig, I.: Lecture notes on automated logical reasoning.
\newblock https://www.slideshare.net/pvcpvc9/lecture17-31382688 (Accessed 2018)

\bibitem{edelkamp2001mips}
Edelkamp, S., Helmert, M.: Mips: The model-checking integrated planning system.
\newblock AI magazine \textbf{22}(3), 67 (2001)

\bibitem{erol1994umcp}
Erol, K., Hendler, J.A., Nau, D.S.: {UMCP}: A sound and complete procedure for
  hierarchical task-network planning.
\newblock In: AIPS, vol.~94, pp. 249--254 (1994)

\bibitem{Ghallab2016}
Ghallab, M., Nau, D., Traverso, P.: Automated Planning and Acting, 1st edn.
\newblock Cambridge University Press, New York, NY, USA (2016)

\bibitem{giunchiglia1999planning}
Giunchiglia, F., Traverso, P.: Planning as model checking.
\newblock In: European Conference on Planning, pp. 1--20. Springer (1999)

\bibitem{hoare1978communicating}
Hoare, C.A.R.: Communicating sequential processes.
\newblock Communications of the ACM \textbf{21}(8), 666--677 (1978)

\bibitem{Huth2004}
Huth, M., Ryan, M.: Logic in Computer Science: Modelling and Reasoning About
  Systems.
\newblock Cambridge University Press, New York, NY, USA (2004)

\bibitem{wizardofaus2018}
Institute, D.S.: Wizard of aus 2018 an autonomous warrior 2018 trial.
\newblock
  http://www.defencescienceinstitute.com/2017/11/06/wizard-aus-2018-autonomous-warrior-2018-trial/
  (2018)

\bibitem{remus100}
Institution, W.H.O.: Remus 100.
\newblock http://www.whoi.edu/main/remus100 (2018)

\bibitem{janson2015fast}
Janson, L., Schmerling, E., Clark, A., Pavone, M.: Fast marching tree: A fast
  marching sampling-based method for optimal motion planning in many
  dimensions.
\newblock The International journal of robotics research \textbf{34}(7),
  883--921 (2015)

\bibitem{Johnson2016}
Johnson, B., Roberts, M., Apker, T., Aha, D.W.: Goal reasoning with informative
  expectations.
\newblock In: ICAPS Workshop. London, UK (2016)

\bibitem{Kondrakunta2017}
Kondrakunta, S., Cox, M.T.: Autonomous goal selection operations for
  agent-based architectures.
\newblock In: IJCAI GRW. Melbourne, Australia (2017)

\bibitem{KressGazit2009}
Kress-Gazit, H., Fainekos, G.E., Pappas, G.J.: Temporal-logic-based reactive
  mission and motion planning.
\newblock IEEE Transactions on Robotics \textbf{25}(6), 1370--1381 (2009)

\bibitem{Kuter2005AHT}
Kuter, U., Nau, D.S., Pistore, M., Traverso, P.: A hierarchical task-network
  planner based on symbolic model checking.
\newblock In: ICAPS (2005)

\bibitem{lee1967foundations}
Lee, E.B., Markus, L.: Foundations of optimal control theory.
\newblock Tech. rep., Minnesota University Minneapolis Center for Control
  Sciences (1967)

\bibitem{newman2008moos}
Newman, P.M.: {MOOS}-mission orientated operating suite  (2008)

\bibitem{mh370search}
news.com.au: New ocean infinity search for mh370 encountering big problems.
\newblock
  http://www.news.com.au/travel/travel-updates/incidents/new-ocean-infinity-search-for-mh370-encountering-big-problems/news-story/89efc71d393585bb3efc1fcd10b765aa
  (2018)

\bibitem{PCP13}
Paisner, M., Cox, M.T., Perlis, D.: Symbolic anomaly detection and assessment
  using growing neural gas.
\newblock In: Proc.\ 25th {IEEE} International Conference on Tools with
  Artificial Intelligence, pp. 175--181. {IEEE} Computer Society (2013)

\bibitem{Pnueli1977}
Pnueli, A.: The temporal logic of programs.
\newblock In: Proceedings of the 18th Annual Symposium on Foundations of
  Computer Science, SFCS '77, pp. 46--57. IEEE Computer Society, Washington,
  DC, USA (1977)

\bibitem{RabideauCM11}
Rabideau, G., Chien, S.A., McLaren, D.: Tractable goal selection for embedded
  systems with oversubscribed resources.
\newblock {JACIC} \textbf{8}(5), 151--169 (2011)

\bibitem{RobertsAJAWA15}
Roberts, M., Apker, T., Johnston, B., Auslander, B., Wellman, B., Aha, D.W.:
  Coordinating robot teams for disaster relief.
\newblock In: Proceedings of the Twenty-Eighth International Florida Artificial
  Intelligence Research Society Conference, {FLAIRS} 2015, Hollywood, Florida.
  May 18-20, 2015., pp. 366--371 (2015)

\bibitem{shivashankar2012hierarchical}
Shivashankar, V., Kuter, U., Nau, D., Alford, R.: A hierarchical goal-based
  formalism and algorithm for single-agent planning.
\newblock In: Proceedings of the 11th International Conference on Autonomous
  Agents and Multiagent Systems-Volume 2, pp. 981--988. International
  Foundation for Autonomous Agents and Multiagent Systems (2012)

\bibitem{Shivishankarthesis}
Shivishankar, V.: Hierarchical goal network planning: Formalisms and algorithms
  for planning and acting.
\newblock Ph.D. thesis, PhD thesis, Department of Computer Science, University
  of Maryland College Park (2015)

\bibitem{sun2008}
Sun, J., Liu, Y., Dong, J.S.: Model checking csp revisited: Introducing a
  process analysis toolkit.
\newblock In: T.~Margaria, B.~Steffen (eds.) Leveraging Applications of Formal
  Methods, Verification and Validation, pp. 307--322. Springer Berlin
  Heidelberg, Berlin, Heidelberg (2008)

\bibitem{SLD09}
Sun, J., Liu, Y., Dong, J.S., Pang, J.: {PAT}: Towards flexible verification
  under fairness.
\newblock pp. 709--714. Springer (2009)

\bibitem{Tinnemeier2008}
Tinnemeier, N.A.M., Dastani, M., Meyer, J.J.C.: Goal Selection Strategies for
  Rational Agents, pp. 54--70.
\newblock Springer Berlin Heidelberg, Berlin, Heidelberg (2008)

\bibitem{valmari1996state}
Valmari, A.: The state explosion problem.
\newblock In: Advanced Course on Petri Nets, pp. 429--528. Springer (1996)

\bibitem{Wilson2016}
Wilson, M.A., McMahon, J., Wolek, A., Aha, D.W., Houston, B.: Toward goal
  reasoning for autonomous underwater vehicles: Responding to unexpected
  agents.
\newblock In: 25th International Joint Conference on Artificial Intelligence
  (IJCAI) Workshop on Goal Reasoning. New York, NY (2016)

\end{thebibliography}


\end{document}